\documentclass[letterpaper, 10 pt, conference]{ieeeconf}  
\newtheorem{theorem}{Theorem}

\newtheorem{definition}{Definition}

\usepackage{graphicx}
\usepackage{amsmath}
\usepackage{amssymb}
\usepackage{subcaption}
\usepackage[colorlinks=true,allcolors=blue]{hyperref}
\usepackage{cleveref}
\usepackage{stfloats}
\usepackage{mathrsfs}

\usepackage{pgfplots}
\pgfplotsset{compat=1.12}

\usepackage{pgfplotstable}
\usepackage{caption}
\usepackage{tikz}
\usepackage{adjustbox}
\IEEEoverridecommandlockouts                              
\title{\LARGE \bf
Verifiable Safety Q-Filters via Hamilton-Jacobi Reachability and Multiplicative Q-Networks
}

\author{Jiaxing Li$^{*}$, Hanjiang Hu$^{*}$, Yujie Yang$^{*}$, Changliu Liu
\thanks{$^{*}$ indicates equal contribution. }
\thanks{Authors are with the Robotics Institute, Carnegie Mellon University, USA, Pittsburgh, PA, 15213, USA.  Work done while Y. Yang was visiting Carnegie Mellon University. {\tt\small \{jiaxingl, hanjianh, yujieyan, cliu6\}@andrew.cmu.edu}.}%
}

\begin{document}

\maketitle
\thispagestyle{empty}
\pagestyle{empty}

\begin{abstract}

Recent learning-based safety filters have outperformed conventional methods, such as hand-crafted Control Barrier Functions (CBFs), by effectively adapting to complex constraints. However, these learning-based approaches lack formal safety guarantees. In this work, we introduce a verifiable model-free safety filter based on Hamilton-Jacobi reachability analysis. Our primary contributions include: 1) extending verifiable self-consistency properties for Q value functions, 2) proposing a multiplicative Q-network structure to mitigate zero-sublevel-set shrinkage issues, and 3) developing a verification pipeline capable of soundly verifying these self-consistency properties. Our proposed approach successfully synthesizes formally verified, model-free safety certificates across four standard safe-control benchmarks.

\end{abstract}

\section{Introduction}
Ensuring the safety of control systems is paramount, especially in safety-critical applications. 
Safety filters are mechanisms designed to ensure that systems avoid unsafe states by monitoring the system's safety risk and intervening when necessary by modifying the control inputs. 
Several value-based methods have been developed to design the monitoring mechanism of these safety filters~\cite{dawson2023safe,wei2019safe}. 
Hamilton-Jacobi (HJ) reachability analysis \cite{Ganai2024} is a widely used approach to compute these value functions. Essentially, it obtains the value function by computing the backward reachable set of the constraint set given the system dynamics and control limits. 
While effective, conventional methods like HJ reachability analysis face challenges related to scalability and generalizability, particularly in complex or uncertain environments.

To address these limitations, learning-based approaches have emerged, leveraging machine learning techniques to construct safety filters. Deep reinforcement learning (RL) is applied in a recent work \cite{Chow2020} to learn a safety critic from a binary reward signal that indicates the safety risk of the states. It can also be combined with conventional methods, e.g., Rubies-Royo \cite{Rubies-Royo2019} applied supervised learning to obtain the HJ reachability value function by approximately solving the safety Bellman equation. These learning-based methods can adapt to a wide range of scenarios and often outperform traditional methods in terms of flexibility and performance.
However, a significant drawback of learning-based approaches is the lack of theoretical guarantees. This is caused by the black box nature of neural networks,
the learned models may not be strictly safe and possibly result in catastrophic failure. Furthermore, certain learning-based methods~\cite{Yang2024} learn a safety state value function that requires knowledge of the system dynamics for online filtering, which can introduce biases and computational overhead.

In this paper, we bridge the gap between learning-based Hamilton-Jacobi (HJ) reachability with model-free filtering and formal safety guarantees by introducing a verifiable framework for the action-value functions.
Directly learning an action-value function poses two major challenges: (i) it demands a more expressive model due to the expanded input space, and (ii) it complicates formal verification and requires new verifiable conditions.
To address these challenges, we propose a novel multiplicative Q-network that simultaneously enhances expressiveness and facilitates efficient verification. Additionally, we introduce two verifiable conditions that prove the safety properties of the action-value function.
Our method follows a post-hoc verification approach: after training, we verify the learned Q-network and then guide further fine-tuning using identified counterexamples. A common issue in verification-guided training is the collapse of the zero-sublevel set over iterations \cite{Yang2024}. Our proposed multiplicative Q-network successfully mitigates this problem, preserving a reasonable zero-sublevel set size throughout the training process.

Our contributions can be summarized as follows.
\begin{itemize}
    \item We propose verifiable conditions for the learned action-value function.
    \item We propose an expressive multiplicative Q network architecture that mitigates the zero-sublevel-set shrinkage problem during verification-guided training.
    \item We develop a verification pipeline to soundly verify the proposed self-consistency conditions for the control-dependent safe invariant set represented by an action-value function.
    
    
\end{itemize}

\section{Related Work}

Hamilton-Jacobi (HJ) reachability analysis has served as a foundational framework for certifying the safety of dynamical systems. There always exists a control strategy that keeps the system trajectory inside the zero-sublevel set of the HJ value function. This definition mirrors the notion of a control invariant set from model predictive control (MPC)~\cite{Rawlings2017}, where admissible inputs exist to keep the state within a safe set. It also aligns with safety sets in control barrier functions (CBFs)~\cite{ames2019control,xiao2021high,choi2021robust}, and the forward invariant sets in the Safe Set Algorithm (SSA)~\cite{liu2014control,zhao2023safety}.

However, conventional HJB-based approaches face severe scalability challenges in high-dimensional systems ~\cite{mitchell2005time,mitchell2008flexible}. To address this, Fisac et al.~\cite{fisac2019bridging} proposed bridging HJ safety analysis with Reinforcement Learning (RL) by approximating the HJ value function via Q-learning. Specifically, they introduced a discounted Hamilton-Jacobi-Bellman (HJB) formulation that ensures contraction properties that facilitate stable learning using neural networks.

The learned Q-value function can then be utilized to construct a model-free safety filter~\cite{fisac2019bridging}. However, training a Q-value function using discounted loss does not guarantee that the resulting model certifiably satisfies true safety properties. Moreover, due to approximation errors inherent in neural networks, the resulting safe set may include unsafe states. As a result, while the trained Q-function may be practically useful, it is not inherently suitable for rigorous safety certification.

To address this issue, Yang et al.~\cite{Yang2024} proposed a formal verification framework that certifies the safety of neural network-based value functions. They formulated safety verification as a Mixed-Integer Linear Programming (MILP) problem,
allowing direct checking of two key properties: constraint satisfaction and forward invariance. However, Yang et al.'s work focuses on verifying a state-value function \(V^\pi(x)\), which requires a dynamic model for online filtering.
Sue et al.~\cite{Sue2024QlearningSafetyFilter} further highlighted the limitations of such model-based safety filters.
On the other hand, optimization-based verification methods, such as MIPVerify~\cite{tjeng2019evaluating}, have demonstrated strong capabilities in certifying neural network properties by encoding ReLU activations and constraints into Mixed-Integer Programs. These methods provide soundness and completeness guarantees. These motivate us to develop model-free safety filters based on action-value functions and leverage MIQCP-based verification in the certification.


\section{Problem Formulation}

Consider a discrete-time deterministic dynamical system with state \(x_t \in X \subset \mathbb{R}^m\) and control input \(u_t \in U \subset \mathbb{R}^n\), governed by the dynamics:
\begin{equation}
    \label{equ:discrete_dynamics}
    x_{t+1} = f(x_t, u_t),
\end{equation}
where \(U\) is assumed to be compact and linearly constrained, and the dynamics function \(f: X \times U \to X\) is assumed to be bounded and Lipschitz continuous.  
Given the discrete-time dynamics, a control policy \(\pi: X \to U\) maps states to inputs, and the closed-loop dynamics are denoted by \(f_\pi(x) = f(x, \pi(x))\).

There is a user-defined constraint function \(h: \mathbb{R}^m \to \mathbb{R}\), where a state is considered violating constraints if \(h(x) > 0\), and non-violating otherwise. It is assumed that $h$ is piecewise smooth and its zero-sublevel set is closed and connected.

\begin{definition}[Safe Invariant Set]
    \label{def:feasible_region}
    A set \(\mathcal{S} \subseteq \mathcal{X}\) is a safe invariant set if for all \(x \in \mathcal{S}\), \(h(x) \leq 0\) and there exists \(u \in \mathcal{U}\) such that:
    \begin{equation}
        f(x, u) \in \mathcal{S}.
    \end{equation}
\end{definition}


\begin{definition}[Maximum Safe Invariant Set]
    \label{def:max_feasible_region}
    The maximum safe invariant set \(\mathcal{S}_{\mathrm{max}}\) is:
    \begin{equation}
        \mathcal{S}_{\mathrm{max}} = \bigcup \left\{ \mathcal{S} \subseteq \mathcal{X} \mid \mathcal{S} \text{ is a safe invariant set} \right\}.
    \end{equation}
\end{definition}

The maximum safe invariant set can be characterized as the zero-sublevel set of the following state-value function under optimal policy \cite{tonkensrefining}:
\begin{equation}
    V(x) = \min_{\pi} \max_{t \geq 0} h(x_t),
\end{equation}
where \(x_{t+1} = f(x_{t}, \pi(x_{t}))\), and:
\begin{equation}
    \mathcal{S}_{\mathrm{max}} = \{ x \mid V(x)\leq 0 \}.
\end{equation}

The value function satisfies the discrete-time Hamilton-Jacobi-Bellman (HJB) variational inequality:
\begin{equation}
    \label{equ:hjb_inequality}
    V(x) = \max \left\{ h(x), \min_{u} V(f(x, u)) \right\}.
\end{equation}

After obtaining the HJ value function, the set of safe control inputs at a given state can be defined as follows.

\begin{definition}[Safe Control Set]
    \label{def:v_to_safe_u}
    Given HJ value function \(V(x)\) and state \(x\), the safe control set \(\mathcal{U}_s(x)\) is:
    \begin{equation}
        \mathcal{U}_s(x) := \left\{ u \in \mathcal{U} \mid V(f(x, u)) \leq 0 \right\}.
    \end{equation}
\end{definition}

While this framework provides a principled way to describe the invariantly safe states and safe controls, solving \Cref{equ:hjb_inequality} exactly is generally intractable for complex systems, and \cite{fisac2019bridging} proposed the following discounted HJB formulation:
\begin{equation}
    \label{equ:pretrain_target}
    \begin{aligned}
        Q^\pi(x, u) &= (1-\gamma) h(x) \\
        &\quad + \gamma \max\{ h(x), Q^\pi(f(x, u), \pi(f(x, u))) \},
    \end{aligned}
\end{equation}
where \(\gamma \in (0,1)\) is the discount factor, and the learned value function $Q(x, u)$ can be applied as a safety filter:
\begin{definition}[Model-Free Safety Filter]\label{def: model-free safety filter}
An action-value function $Q(x, u)$ is a model-free safety filter with its induced safe control set
\[
\mathcal{U}_{\text{safe}}(x_{t}) := \{ u_{t} \in \mathcal{U} \mid Q(x_{t}, u_{t}) \leq 0 \}.
\]  
if for any initial state $x_{0}\in \mathcal{X}$ and initial control $u_{0}\in \mathcal{U}$ such that $Q(x_{0}, u_{0})<0$, 
any trajectory \( \{x_t\}_{t=0}^{\infty} \) generated by applying actions \( u_t \in \mathcal{U}_{\text{safe}}(x_t) \) satisfies
\[
h(x_t) \leq 0 \quad \text{for all } t \geq 0.
\]
\end{definition}

However, the learned Q value function in \eqref{equ:pretrain_target} does not meet the requirement in the definition due to the discounting factor. This raises a fundamental question:  
\textbf{How can we learn a model-free safety filter with certifiable safety guarantees?}

\begin{figure*}[!ht]
    \centering
    \includegraphics[width=0.8\textwidth]{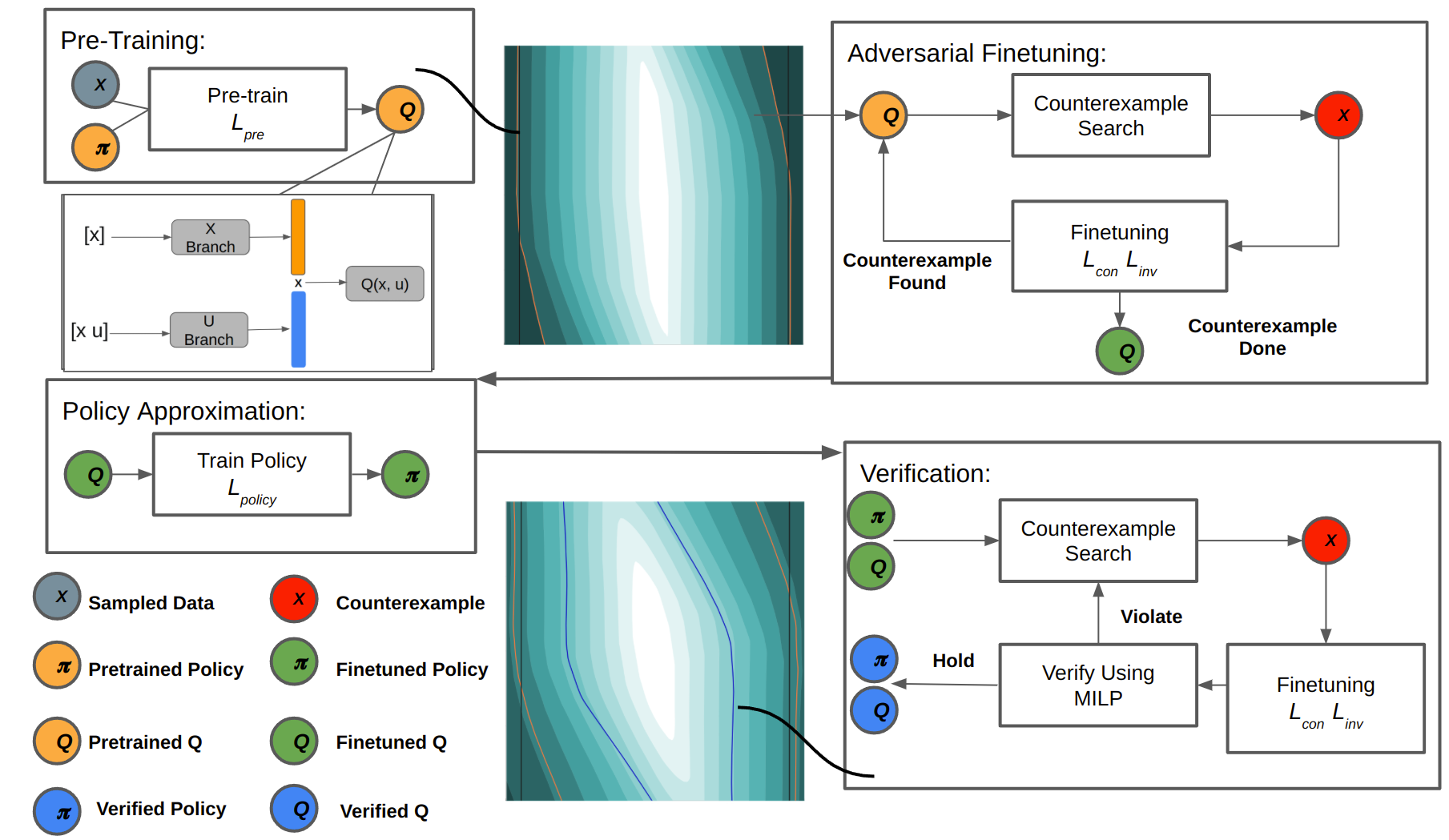}
    \caption{Illustration of the verification pipeline toward a certified model-free safety filter.  The pipeline consists of four stages: pre-training, adversarial fine-tuning, policy approximation, and verification. The Q-function is first pre-trained on a pretrained policy, then fine-tuned using counterexamples that violate the sufficient conditions under a loose interval-based bound. This loop repeats until no new counterexamples are found. A policy is then trained to minimize the fine-tuned Q-function as a conservative approximation. Finally, verification is performed using the approximated policy and Q-function to identify violations under a tighter conservative bound.}
    \label{fig:VerificationPipeline}
\end{figure*}

\section{Methodology}
This work presents a structured framework for synthesizing a verifiable model-free safety filter based on a learned HJ reachability Q-value function. The methodology contains four key steps: we first introduce sufficient conditions for the Q function to be a valid model-free safety filter. Then we talk about the parameterization of the Q function, followed by the method to certify the learned Q function with respect to the sufficient conditions, as well as the method to improve the error-prone Q function. 
\subsection{Sufficient Conditions}

\begin{definition}[Control-dependent Safe Set]
\label{def:con_safe_region}
Given dynamics model $f$ and constraint $h$, $\mathcal{I} \subseteq \mathcal{X} \times \mathcal{U}$ is a control-dependent safe set if
\begin{equation*}
\forall (x_0, u_0) \in \mathcal{I}, \ \exists \pi, \ \text{s.t.} \ h(x_t) < 0, \ \forall t \in \mathbb{N},
\end{equation*}
where $x_1=f(x_0,u_0)$ and $x_{t+1}=f_{\pi}(x_t),\forall t\ge1$.
\end{definition}

Implied by the above definition, for any state-control pair $(x_s, u_s)$ in the safe set $\mathcal{I}$, there exists a safe policy $\pi$ that maintains long-term safety of the system by keeping all subsequent states within the control-dependent safe set.


\begin{theorem}[Sufficient Conditions for $Q(x, u)$]
    \label{theorem:self_con_Q}
An action-value function \( Q(x, u) \) is a valid model-free safety filter according to \cref{def: model-free safety filter} if the following two conditions are satisfied:
\begin{enumerate}
    \item \textbf{Constraint Satisfaction:} For all \(Q(x, u) \leq 0\), we have
    \[
    h(x) \leq 0.
    \]
    \item \textbf{Forward Invariance:} For all \(Q(x, u) \leq 0\), there exists \( u' \in U \) such that
    \[
    Q(f(x, u), u') \leq 0.
    \]
\end{enumerate}
Then its zero-sublevel-set  \( \mathcal{I} :=\{(x,u)\mid Q(x,u)\leq 0\}\) is a control-dependent safe set according to \cref{def:con_safe_region}. 
\end{theorem}

\begin{proof}
Assume that \( Q(x, u) \) satisfies constraint satisfaction and forward invariance as stated above. Then, starting from any \((x_0, u_0) \in \mathcal{I}\), we can recursively construct a trajectory $\xi := \{(x_t, u_t)\}_{t=0}^{\infty}$, such that \( Q(x_t, u_t) \leq 0 \) and \( h(x_t) \leq 0 \) for all \( t \geq 0 \). Specifically, given \((x_t, u_t)\), the existence of \( u_{t+1} \) satisfying \( Q(f(x_t, u_t), u_{t+1}) \leq 0 \) ensures the trajectory can be continued indefinitely. Then $Q(x,u)$ is a valid model-free safety filter by \cref{def: model-free safety filter}. Consequently, the zero-sublevel-set \(\mathcal{I}\) remains invariant under the closed-loop dynamics. Thus, \( \mathcal{I} \) satisfies the definition of a control-dependent safe set in \Cref{def:con_safe_region}.
\end{proof}

Constraint satisfaction ensures that the safety value $Q(x, u)$ can only be smaller than 0 when current state $x$ is safe,
and forward invariance means in order for the current state action pair $(x, u)$ to be safe, there must exist an action $u'$ for the next state $x' = f(x, u)$ that can maintain the system inside the safe set.
Furthermore, $\exists u'\in U$, $Q(f(x, u), u') \leq 0$ can be ensured if its sufficient condition $\min_{u'}{Q(f(x, u), u')}\le0$ holds.

\subsection{Multiplicative Q network}
We parameterize the Q value function by observing its symbolic relationships with state and control:
\begin{align}
    \label{parameter_Q}
    Q(x, u) &= \max \{h(x), V(f(x, u))\}\\
    & = V(x + \dot{x}\cdot\Delta t) \quad \text{(}V(f(x, u)) > h(x) \text{)}\\ 
    & = V(x) + L_{f}V(x)\cdot \Delta t + L_{g}V(x)\cdot \Delta t \cdot u
\end{align}
We focus on the formulation of $Q(x, u)$ when $V(f(x, u)) > h(x)$, since the complementary case $Q(x, u)=h(x)$ is trivial. The Lie derivative $L_{f}V(x)$ can be viewed as the inertia that directly affects the magnitude of the value function, and $L_{g}V(x)$ is the most unsafe control direction that provides a decomposition of the control effort. It is worth noting the existence of the bilinear term $L_{g}V(x)\cdot \Delta t \cdot u$. It is difficult for a fully connected neural network to capture the multiplicative interaction \cite{Li2016FactorizedBM}. Thus, instead of naively concatenating or adding the state-related branch (X branch) and control-related branch (U branch), we propose a novel multiplicative Q network where a bilinear quadratic multiplication is adopted. 
\Cref{fig:VerificationPipeline} shows the structure of the proposed Q network, where the state is passed through the X branch to obtain an independent feature representation, while the control is passed through the U branch with the state to obtain the control embedding, which is dependent on the state. Then, a vector inner product is applied to combine the two embeddings into the final Q value. 

\subsection{Certification of Q}
To verify the two sufficient conditions in \cref{theorem:self_con_Q}, we reformulate them as feasibility problems. Let $Y_\text{con}, Y_\text{inv}$ encode the constraint satisfaction and forward invariance conditions in \cref{theorem:self_con_Q}:
\begin{align}
    &\mathrm{Y}_{\mathrm{con}}(x, u) = \begin{bmatrix} Q(x, u) \\ h(x) \end{bmatrix} \\
    &\mathrm{Y}_{\mathrm{inv}}(x, u) = \begin{bmatrix} Q(x, u) \\ \min_{u'}Q(f(x, u), u') \end{bmatrix} 
\end{align}
\begin{definition}[Feasibility Conditions for Verification]\label{verify_Q_problem}
Define $\mathrm{Y'}(x, u) = \{[y_1, y_2] \in \mathbb{R}^2 \mid y_1 < 0, y_2 > 0 \}$. The Q function satisfies the sufficient conditions in \cref{theorem:self_con_Q} if there is no solution to the following problem,
\begin{align}
    &\text{ Find } (x_c, u_c) \in \mathcal{X} \times \mathcal{U} \text{, s.t.} \notag \\
    &\quad \mathrm{Y}_{\mathrm{con}}(x_c, u_c) \in \mathrm{Y}' \; \textbf{or} \; \mathrm{Y}_{\mathrm{inv}}(x_c, u_c) \in \mathrm{Y}'.
\end{align}
Otherwise, the solution $(x_c,u_c)$ is a counter example.
\end{definition}

In order to verify the properties soundly, we propose to use a parameterized policy network $\pi_{\phi}$ to approximate the value of $\min_{u'}Q(x, u')$ used in $Y_{\text{inv}}$.
\begin{definition}[Lower Bound Approximation]
    \label{theorem:certified_Q_bound}
The forward invariance condition in \cref{theorem:self_con_Q} can be soundly certified by approximating the lower bound of \(Q\) with a policy network \(\pi_{\phi}: \mathcal{X} \to \mathcal{U}\), trained to minimize:
\[
\mathcal{L}_{policy} = Q(x, \pi_{\phi}(x)).
\]
Soundness is ensured because if \(Q(x', \pi_{\phi}(x')) \leq 0\), then \(\min_{u'} Q(x', u') \leq 0\).
\end{definition}

The verification of a fully connected neural network can be formulated as a Mixed-Integer Quadratically Constrained Programming (MIQCP) problem as shown in \eqref{equ:MIQCP}, where $NN(x) = \sigma(W_n (\sigma(W_{n-1} (\dots \sigma(W_1 x + b_1) \dots ) + b_{n-1}) ) + b_n)$ is piecewise linear. To verify the proposed multiplicative Q network, we added a quadratic constraint that is $y = (Z_{1}(x)\cdot Z_{2}(x, u))$, where $Z_{1}(x) = \sigma(W_n (\sigma(W_{n-1} (\dots \sigma(W_1 x + b_1) \dots ) + b_{n-1}) ) + b_n)$  $Z_{2}(x, u) = \sigma(W_n (\sigma(W_{n-1} (\dots \sigma(W_1 (x, u) + b_1) \dots ) + b_{n-1}) ) + b_n)$. By adding such constraint, the verification is formulated into the following MIQCP problem:
\begin{equation}
    \label{equ:MIQCP}
    \text{find } x,u \quad \text{s.t. } (x,u) \in \mathcal{I}, y \in \mathrm{Y'}, y = (Z_{1}(x)\cdot Z_{2}(x, u)),
\end{equation}
We further illustrate the proposed MIQCP problem with an example focusing on the constraint satisfaction condition for a ReLU activated multiplicative network:


\begin{subequations}
\begin{align}
\label{equ:miqcp_example}
\text{find} \quad & (x, u) \tag{20a} \\
\text{s.t.} \quad & (x, u) \in (\mathcal{X} \times \mathcal{U}) \tag{20b} \\
& z^{\mathscr{X}}_0 = x, \quad z^{\mathscr{U}}_0 = (x, u) \tag{20c} \\
& z^{c}_j = W^{c}_j z^{c}_{j-1} + b^{c}_j, \quad j = \{1, \dots, n^{c}\}, \tag{20d} \\
& \text{if } \hat{l}^{c}_{j,k} \geq 0,\quad z^{c}_{j,k} = \hat{z}^{c}_{j,k}, \tag{20e}\\
&\quad j = \{1, \dots, n^{c}-1\}, k = \{1, \dots, d^{c}_j\},\notag\\& \quad\forall c \in \{\mathscr{X, U}\}\tag{20f} \\
& \text{if } \hat{h}^{c}_{j,k} \leq 0,\quad z^{c}_{j,k} = 0, \tag{20g} \\
& \text{otherwise,} \notag \\
& \quad \quad z^{c}_{j,k} \leq \hat{z}^{c}_{j,k}, \tag{20h} \\
& \quad \quad z^{c}_{j,k} \geq 0, \tag{20i} \\
& \quad \quad z^{c}_{j,k} \leq \hat{z}^{c}_{j,k} - \hat{l}^{c}_{j,k}(1 - \delta^{c}_{j,k}),  \delta^{c}_{j,k} \in \{0,1\} \tag{20j} \\
& \quad \quad z^{c}_{j,k} \leq \hat{u}^{c}_{j,k} \delta^{c}_{j,k}, \tag{20k} \\
& Q(x, u) = z^{\mathscr{X}}_{n^x} \cdot z^{\mathscr{U}}_{n^u} \quad  \tag{20l}\\
& Q(x, u) \leq 0 , \tag{20m}\\
& h(x) \geq 0 , \tag{20n}
\end{align}
\end{subequations}
where equations (20c)--(22k) encode the nonlinear \textit{ReLU activation} for each node \( z^{c}_{j, k} \), where \( c \in \{\mathscr{X}, \mathscr{U}\} \) denotes the \textbf{branch} (state or control), \( j \in \{1, \ldots, n_c\} \) indexes the number of \textbf{layers} in branch \( c \), and \( k \in \{1, \ldots, k^{c}_j\} \) indexes the \textbf{nodes} at layer \( j \) of branch \( c \).

And the $\hat{l}^{c}_{j,k}$, $\hat{h}^{c}_{j,k}$ denote the lower and upper pre-activation bounds and it can be calculated by any reachability-based method \cite{Liu2019}, and we use linear symbolic bound propagation \cite{zhang2018efficient,hu2024real,hu2024verification}.
The ReLU function is reformulated as a set of linear constraints using an auxiliary binary variable \( \delta^{c}_{j,k} \), that indicates the activation status. And equation (20l) represents the \textit{quadratic constraint} introduced by the multiplication layer that combines the two branches. The verification of forward invariance condition follows the same constraint encoding and replacing $h(x)$ with $Q(f(x, u), \pi_{\phi}(f(x, u)))$. We used Gurobi 11.0.2 optimizer for solving the above MIQCP problem.

\subsection{Verification Pipeline}
We first pretrain the Q-network using the discounted self-consistency condition~\eqref{equ:pretrain_target} with a pretrained policy \( \pi_n \). Then we apply adversarial finetuning by identifying counterexamples that violate the constraint satisfaction or forward invariance condition (Theorem~\ref{theorem:self_con_Q}). To efficiently approximate \( \min_{u'} Q \), we use its sound relaxation \( \min_{IA} Q \) calculated through interval arithmetic \cite{Liu2019}
which satisfies \( \min_{IA} Q \leq \min_{u'} Q \). Thus, violations detected under \( \min_{IA} Q \) imply violations of the original condition. Finetuning the pretrained Q-network \( Q_{\text{pre}} \) on these counterexamples yields a refined network \( Q_{\text{ft}} \) better suited for certification, the finetuning loss is defined as follows:
\begin{definition}[Finetune Loss]
\label{def:finetune_loss}
Given a non-empty set of constraint satisfaction counterexamples \(\{(x^{(i)}_{\text{con}}, u^{(i)}_{\text{con}})\}_{i=1}^{n_{\text{con}}}\) with \(n_{\text{con}} > 0\), and a non-empty set of invariance counterexamples \(\{(x^{(i)}_{\text{inv}}, u^{(i)}_{\text{inv}})\}_{i=1}^{n_{\text{inv}}}\) with \(n_{\text{inv}} > 0\), the corresponding finetuning loss functions are defined as:
\begin{align*}
\mathcal{L}_\text{con} &= 
- \sum_{i=1}^{n_{\text{con}}} Q(x^{(i)}_{\text{con}}, u^{(i)}_{\text{con}}), \\
\mathcal{L}_\text{inv} &= 
\sum_{i=1}^{n_{\text{inv}}} \left( -Q(x^{(i)}_{\text{inv}}, u^{(i)}_{\text{inv}}) + \min_{u'} Q(f(x^{(i)}_{\text{inv}}, u^{(i)}_{\text{inv}}), u') \right).
\end{align*}
\end{definition}

Meanwhile, a policy \( \pi_\phi \) is trained to approximate the refined lower bound, enhancing adherence to the self-consistency property. We applied the finetuning pipeline proposed by Yang et al. \cite{Yang2024} by applying their boundary-guided backtracking search to find counterexamples.
Finally, with the obtained $\pi_{\phi}$, we formulate the verification problem into the proposed MIQCP problem \ref{equ:MIQCP} and pass it to the Gurobi solver. If a property is violated, a counterexample will be returned, and the Q network is further finetuned on the counterexample before being passed to be verified again.

\begin{table}[t]
   \centering
   \caption{Safe Set Sizes: Finetuned / Verified}
   \begin{tabular}{|c|c|c|}
       \hline
       Task & Baseline & Proposed \\
       \hline
       double integrator & 0.0583 / 0.0000 & 0.7847 / 0.529 \\
       \hline
       2D double integrator & 0.0380 / 0.0000 & 0.6483 / 0.4720 \\
       \hline
       unicycle & 0.0000 / 0.0000 & 0.7680 / 0.6900 \\
       \hline
       robot arm & 0.0000 / 0.0000 & 0.6640 / 0.4957 \\
       \hline
   \end{tabular}
   \label{tab:merged_safe_set}
\end{table}

\begin{figure}[t]
    \centering
    \begin{adjustbox}{width=0.7\columnwidth}
    \begin{tikzpicture}
        \begin{axis}[
            ymin=0.85,
            ymax=1.0,
            ylabel={Safe Control Set Size (\%)},
            symbolic x coords={Pretrain, Finetune, Verify},
            xtick=data,
            xticklabel style={font=\small},
            legend style={
                at={(0.5,-0.25)}, anchor=north, legend columns=2,
                draw=none, fill=none,
                /tikz/every even column/.append style={column sep=0.5cm}
            },
            legend image code/.code={
                \draw[thick] (0cm,-0.1cm) -- (0.4cm,-0.1cm);
            },
            grid=both,
            major grid style={line width=.2pt,draw=gray!30},
        ]

        \addplot+[thick, mark=*, color=blue!80!black] coordinates {
            (Pretrain, 0.9810)
            (Finetune, 0.9822)
            (Verify, 0.8883)
        };

        \addplot+[thick, mark=*, color=red!80!black] coordinates {
            (Pretrain, 0.9981)
            (Finetune, 0.9764)
            (Verify, 0.8876)
        };

        \addplot+[thick, mark=*, color=green!60!black] coordinates {
            (Pretrain, 0.9827)
            (Finetune, 0.9600)
            (Verify, 0.9731)
        };

        \addplot+[thick, mark=*, color=purple!70!black] coordinates {
            (Pretrain, 0.9726)
            (Finetune, 0.9619)
            (Verify, 0.9195)
        };

        \legend{
            Double Integrator,
            2D Double Integrator,
            Unicycle,
            Robot Arm
        }
        \end{axis}
    \end{tikzpicture}
    \end{adjustbox}
    \caption{Safe control set size for each task at different stages of pipeline.}
    \label{fig:safe_control_curve}
    \vspace{-5mm}
\end{figure}
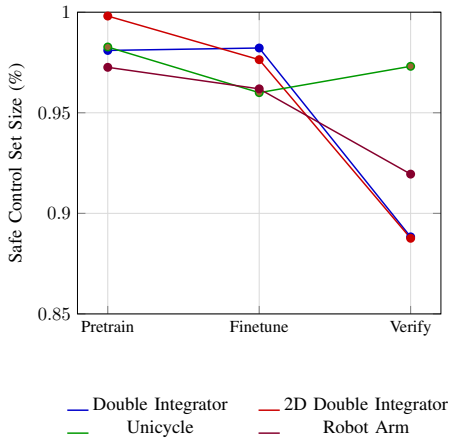

\section{Experiment}

Our experiments aim to answer the following questions:\\
\textbf{Q1}: Can the proposed multiplicative Q-network be formally verified to satisfy the sufficient conditions derived from HJ reachability?\\
\textbf{Q2}: Does the verified control dependent safe set maintain a nontrivial safe set across diverse control tasks?

\subsection{Implementation Details}
We evaluated our verification framework on four safe control tasks: double integrator, 2D double integrator, unicycle, and robot arm, covering both linear and nonlinear dynamics and constraints. The unicycle system has nonlinear dynamics, while the others are linear; constraints are nonlinear except for the double integrator. The input dimensions range from 3 to 6.
The proposed multiplicative Q-network uses separate \(X\)- and \(U\)-branches, each with two hidden layers of 32 neurons, and an 8-dimensional output embedding. The baseline is a fully connected network with two hidden layers of 32 neurons. The proposed model has 2,560 learnable parameters, compared to 6,144 in the baseline.
The runtime experiments were conducted on Ubuntu 22.04 with Intel Xeon Silver $4214$ CPUs and
NVIDIA RTX A4000 GPUs (each with $16$ GB VRAM) based on the verification toolbox ~\cite{wei2024modelverification}.

\subsection{Results}
Our multiplicative Q-network successfully satisfies the 
sufficient conditions of HJ reachability while maintaining a reasonable safe set size.

As shown in \Cref{tab:merged_safe_set}, the proposed network avoids the safe set collapse issue observed in the baseline. After finetuning, the baseline network often collapses to an empty safe set—particularly for the unicycle and robot arm systems—whereas our method preserves safe sets with sizes 0.7603 and 0.6313, respectively. Furthermore, while the baseline network cannot be verified without collapsing, the proposed network consistently maintains a verifiable safe set across all tasks.

\begin{figure*}[!t]
    \centering
    \begin{subfigure}[t]{0.24\textwidth}
        \includegraphics[width=\textwidth]{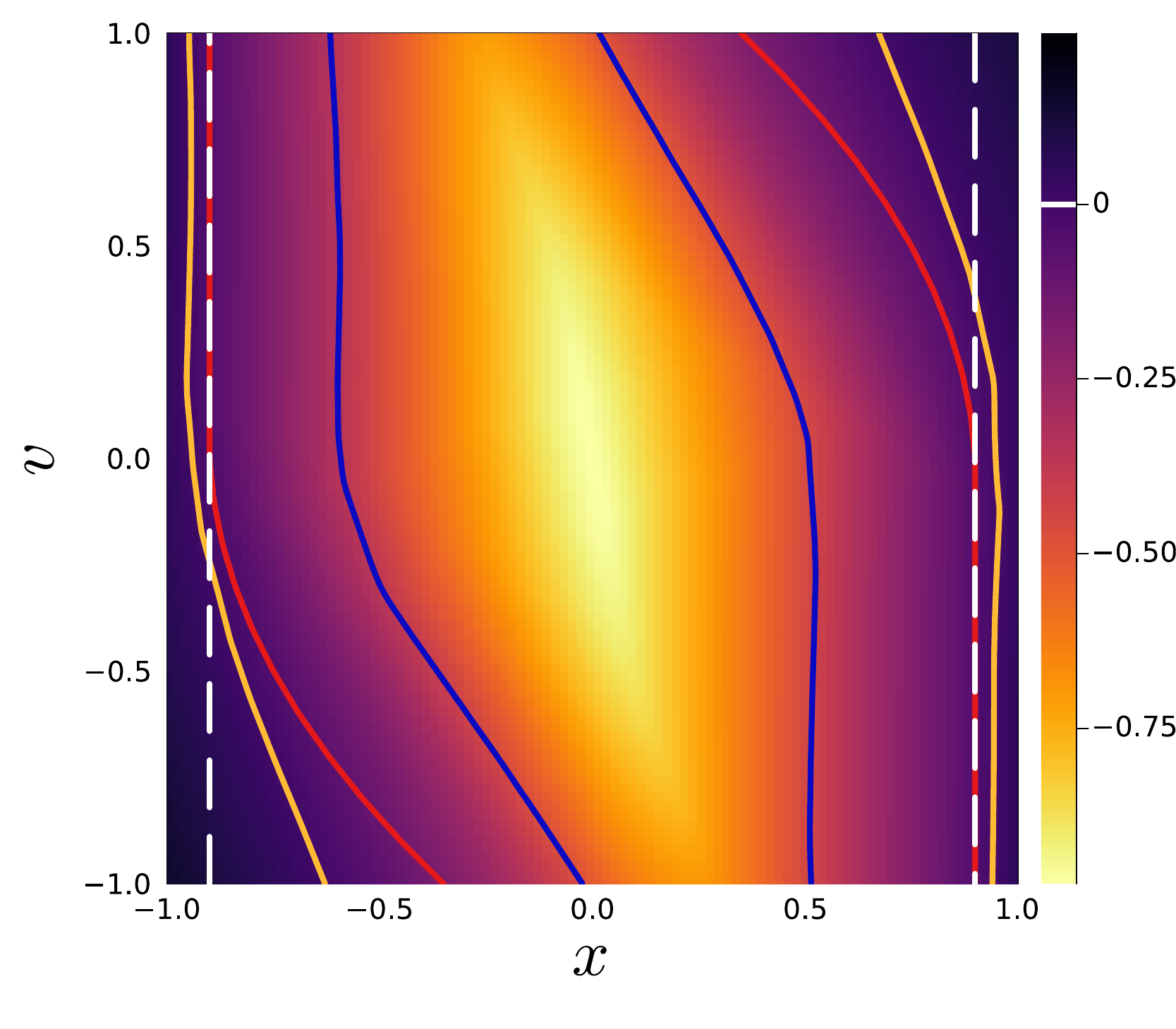}
        \caption{}
        \label{fig:double_integrator_safe_set}
    \end{subfigure}
    \begin{subfigure}[t]{0.24\textwidth}
        \includegraphics[width=\textwidth]{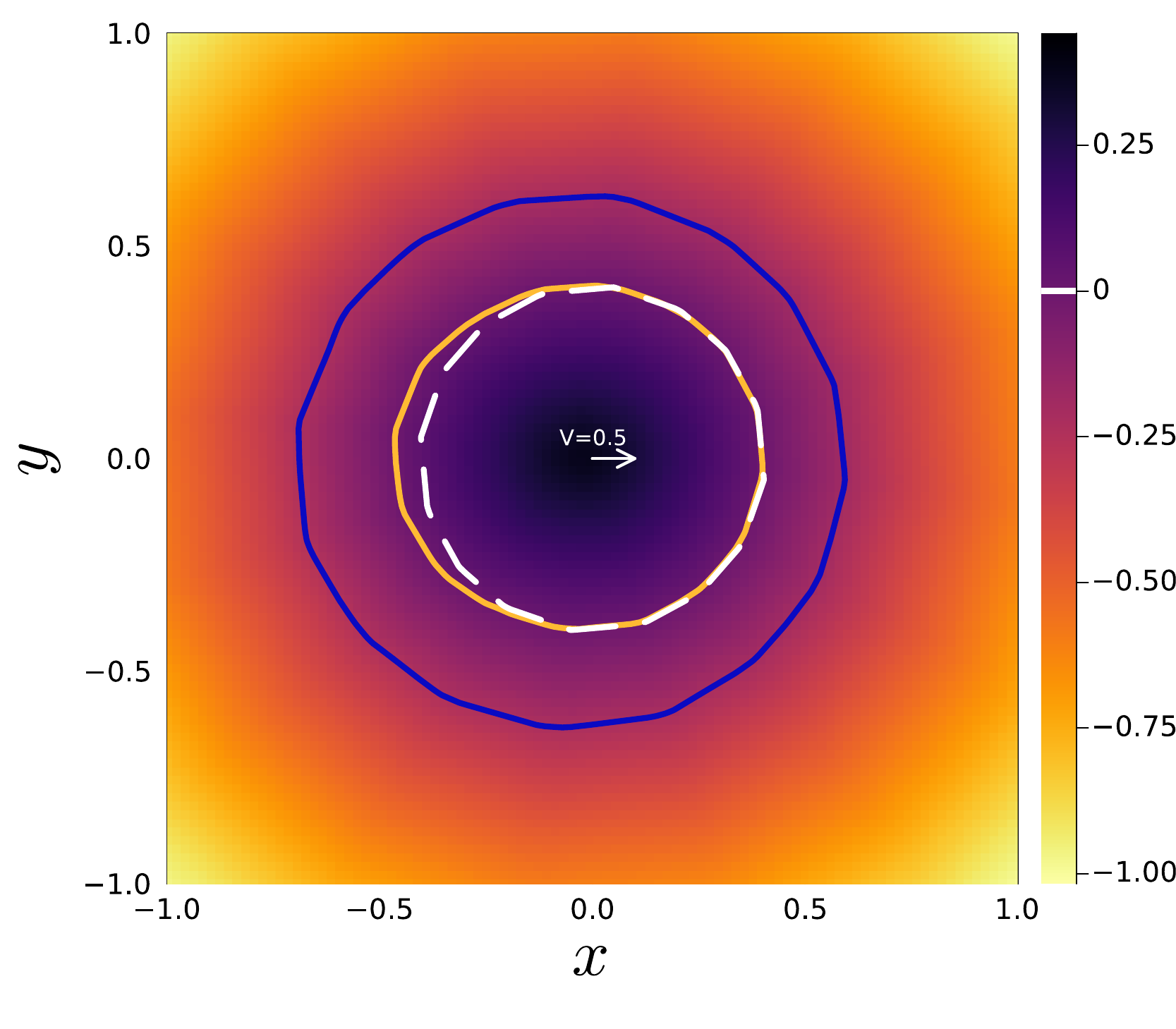}
        \caption{}
        \label{fig:2D_di}
    \end{subfigure}
    \begin{subfigure}[t]{0.24\textwidth}
        \includegraphics[width=\textwidth]{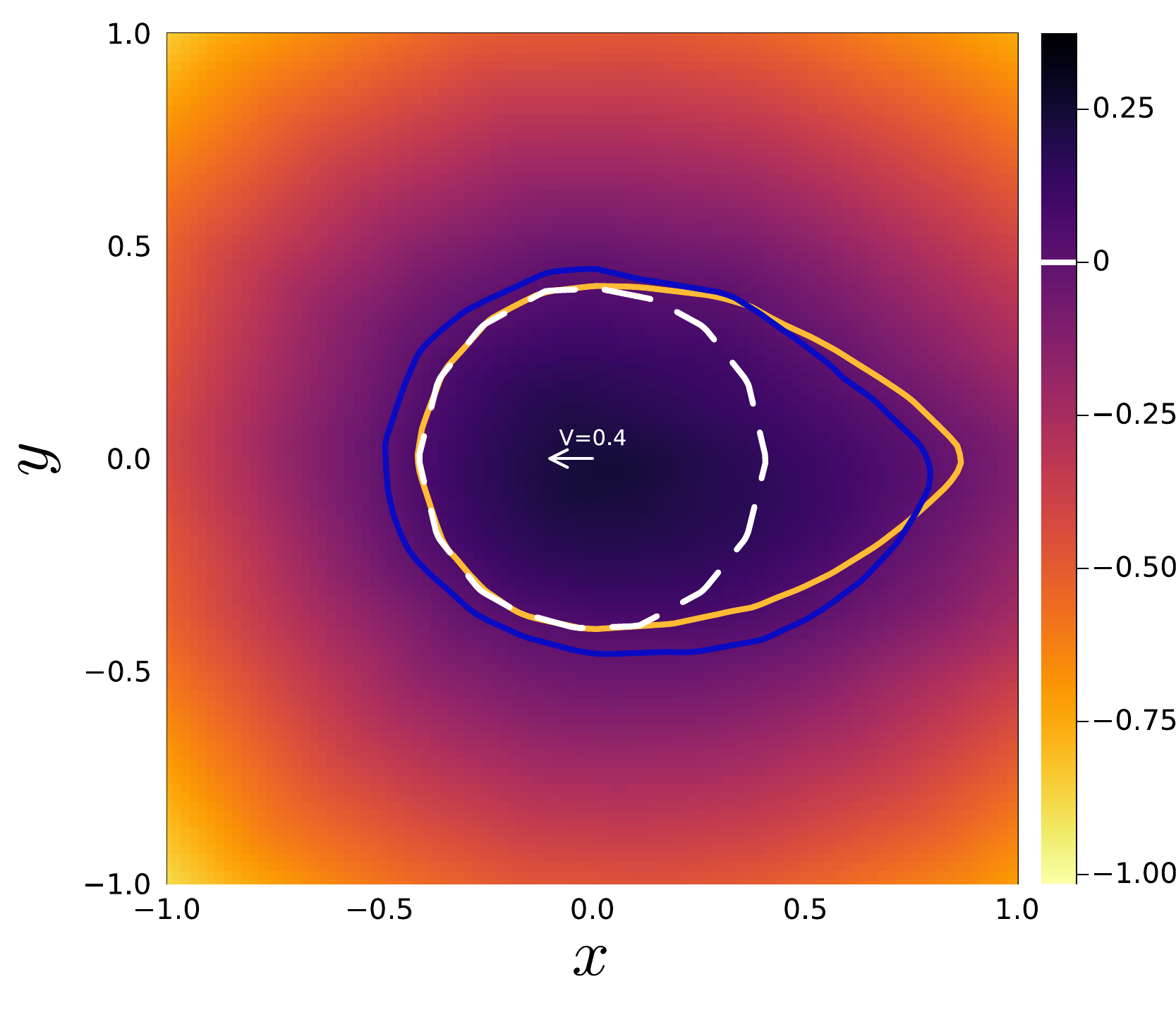}
        \caption{}
        \label{fig:unicycle}
    \end{subfigure}
    \begin{subfigure}[t]{0.24\textwidth}
        \includegraphics[width=\textwidth]{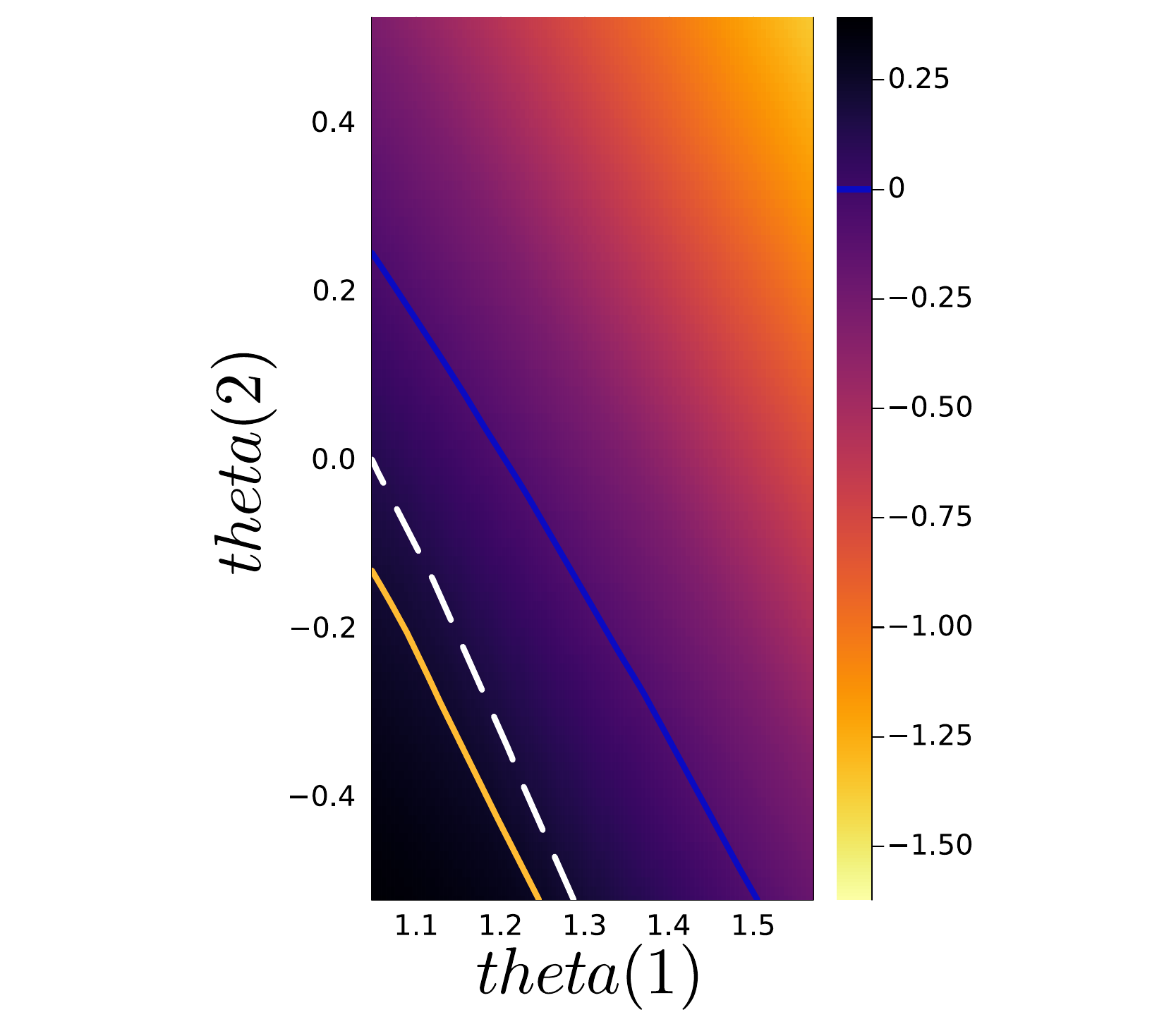}
        \caption{}
        \label{fig:robot_arm}
    \end{subfigure}
    
    \vspace{0.5em}
    
    \includegraphics[width=0.6\textwidth]{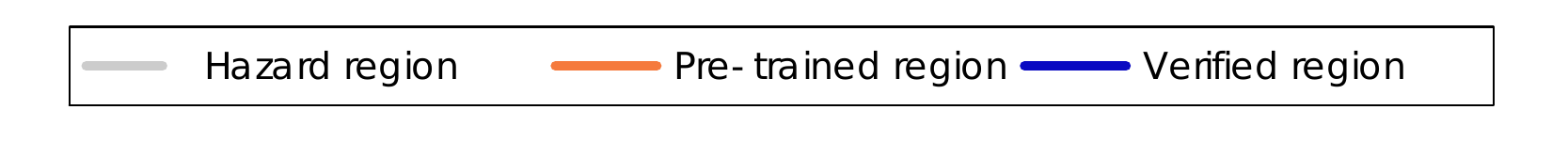}
    
    \caption{Verified safe sets for (a) double integrator, we also demonstrate the ground truth safe set of double integrator with the red contour, (b) 2D double integrator with hazard regions of radius $r=0.4$, (c) unicycle with fixed speed of $0.9$ (direction indicated by arrow), and (d) robot arm with fixed speed.}
    \label{fig:all_safe_sets}
    \vspace{-3mm}
\end{figure*}
In \Cref{fig:double_integrator_safe_set}, we show the verified safe invariant set with the ground truth one in red contour derived from explicitly solving the HJB variational inequality for the double integrator system, which has linear dynamics and constraints. The verified safe invariant set strictly satisfies the constraint satisfaction and forward invariance conditions. However, it still has a certain degree of zero-sublevel-set shrinkage problem compared with the ground truth.



In \Cref{fig:double_integrator_safe_set}, we also highlight the hazardous region \(|d| \geq 0.9\) marked by dashed lines. The pretrained Q-network (orange contour) violates constraint satisfaction by misclassifying part of the hazard as safe, while the verified Q-network (blue contour) corrects this error through optimal control via the trained policy \(\pi_{\phi}\).

\Cref{fig:2D_di,fig:unicycle} display safe sets for the 2D double integrator and unicycle systems in Cartesian space. The heading is fixed at \(0^\circ\) and \(180^\circ\), with speeds of 0.5 m/s and 0.4 m/s. In both cases, the pretrained Q-function incorrectly identifies regions near the center hazard (white dashed circle, radius 0.4) as safe, while the verified Q-function avoids such violations.

For the robot arm system, where the state is \([\theta_1, \theta_2, \dot{\theta}_1, \dot{\theta}_2]\) and control is \([\omega_1, \omega_2]\), the safe set is defined by a distance constraint relative to the base. As shown in \Cref{fig:robot_arm}, the pretrained Q-function overestimates the safe set by including hazardous joint configurations (black boundary), whereas the verified Q-function strictly excludes unsafe regions.

Finally, \Cref{fig:safe_control_curve} shows the evolution of the average safe control set size through the pipeline. The safe set becomes more conservative as counterexamples are incorporated, reflecting improved robustness.

\section{Conclusions}
In this paper, we introduced a verifiable, model-free safety filter grounded in Hamilton-Jacobi reachability. Our proposed multiplicative Q-network enables certification of constraint satisfaction and forward invariance for control-dependent safe sets, while alleviating zero-sublevel-set shrinkage common in prior methods. Experiments validate its effectiveness across both linear and nonlinear systems.
Future work will focus on extending the framework to high-dimensional, continuous control tasks, and on refining safety guarantees by analyzing the conservativeness introduced by interval arithmetic.





\bibliographystyle{plain}
\bibliography{./references}
\end{document}